\documentclass[a4paper,english,british]{article}
\usepackage[T1]{fontenc}
\usepackage[latin9]{inputenc}
\usepackage{babel}
\usepackage{array}
\usepackage{multirow}
\usepackage{amsmath}
\usepackage{amsthm}
\usepackage{amssymb}
\usepackage{graphicx}
\usepackage[numbers]{natbib}
\usepackage[unicode=true]
 {hyperref}

\makeatletter

\pdfpageheight\paperheight
\pdfpagewidth\paperwidth

\providecommand{\tabularnewline}{\\}

\theoremstyle{plain}
\newtheorem{thm}{\protect\theoremname}[section]
  \theoremstyle{plain}
  \newtheorem{prop}[thm]{\protect\propositionname}
  \theoremstyle{plain}
  \newtheorem{lem}[thm]{\protect\lemmaname}

\usepackage{times}
\usepackage{algorithm}
\usepackage{algorithmic}
\usepackage{hyperref}
\usepackage[accepted]{icml2016} 
\usepackage{placeins}

\makeatother

  \addto\captionsbritish{\renewcommand{\lemmaname}{Lemma}}
  \addto\captionsbritish{\renewcommand{\propositionname}{Proposition}}
  \addto\captionsbritish{\renewcommand{\theoremname}{Theorem}}
  \addto\captionsenglish{\renewcommand{\lemmaname}{Lemma}}
  \addto\captionsenglish{\renewcommand{\propositionname}{Proposition}}
  \addto\captionsenglish{\renewcommand{\theoremname}{Theorem}}
  \providecommand{\lemmaname}{Lemma}
  \providecommand{\propositionname}{Proposition}
\providecommand{\theoremname}{Theorem}

\begin{document}
\twocolumn[ \icmltitle{A Kernel Test of Goodness of Fit}
\icmlauthor{Kacper Chwialkowski$^*$}{kacper.chwialkowski@gmail.com}
\icmlauthor{Heiko Strathmann$^*$}{heiko.strathmann@gmail.com}
\icmlauthor{Arthur Gretton}{arthur.gretton@gmail.com}
\icmladdress{Gatsby Unit, University College London, United Kingdom}
\icmlkeywords{kernel methods, goodness-of-fit, Stein's method, statistical testing}
\vskip 0.3in ]

\begin{abstract}
We propose a nonparametric statistical test for goodness-of-fit:
given a set of samples, the test determines how likely it is that
these were generated from a target density function. The measure of
goodness-of-fit is a divergence constructed via Stein's method using
functions from a reproducing kernel Hilbert space. Our test statistic
is based on an empirical estimate of this divergence, taking the form
of a V-statistic in terms of the gradients of the log target density
 and of the kernel. We derive a statistical test, both for i.i.d.
and non-i.i.d. samples, where we estimate the null distribution quantiles
using a wild bootstrap procedure. We apply our test to quantifying
convergence of approximate Markov chain Monte Carlo methods, statistical
model criticism, and evaluating quality of fit vs model complexity
in nonparametric density estimation.
\end{abstract}
\selectlanguage{english}%
\global\long\def\ev{\mathbb{E}}

\selectlanguage{british}%

\section{Introduction}

Statistical tests of goodness-of-fit are a fundamental tool in statistical
analysis, dating back to the test of Kolmogorov and Smirnov \citep{Kolmogorov33,Smirnov48}.
Given a set of samples $\{Z_{i}\}_{i=1}^{n}$ with distribution $Z_{i}\sim q$,
our interest is in whether $q$ matches some reference or target distribution
$p$\foreignlanguage{english}{, which we assume to be only known up
to the normalisation constant. Recently, in the multivariate setting,
}\citet{gorham2015measuring} proposed an elegant measure of sample
quality with respect to a target. This measure is a maximum discrepancy
between empirical sample expectations and target expectations over
a large class of test functions, constructed so as to have zero expectation
over the target distribution by use of a Stein operator. This operator
depends only on the derivative of the $\log q$: thus, the approach
can be applied very generally, as it does not require closed-form
integrals over the target distribution (or numerical approximations
of such integrals). By contrast, many earlier discrepancy measures
require integrals with respect to the target (see below for a review).
This is problematic e.g.~if the intention is to perform benchmarks
for assessing Markov Chain Monte Carlo, since these integrals are
certainly not known to the practitioner.

A challenge in applying the approach of \citeauthor{gorham2015measuring}
is the complexity of the function class used, which results from applying
the Stein operator to the $W^{2,\infty}$ Sobolev space. Thus, their
sample quality measure requires solving a linear program that arises
from a complicated construction of graph Stein discrepancies and geometric
spanners. Their metric furthermore requires access to nontrivial lower
bounds that, despite being provided for log-concave densities, are
a largely open problem otherwise, in particular for multivariate cases.

An important application of a goodness-of-fit measure is in statistical
testing, where it is desired to determine whether the empirical discrepancy
measure is large enough to reject the null hypothesis (that the sample
arises from the target distribution). One approach is to establish
the asymptotic behaviour of the test statistic, and to set a test
threshold at a large quantile of the asymptotic distribution. The
asymptotic behaviour of the $W^{2,\infty}$-Sobolev Stein discrepancies
remains a challenging open problem, due to the complexity of the function
class used. It is not clear how one would compute p-values for this
statistic, or  determine when the goodness of fit test  allows to
accept the null hypothesis at a user-specified test level.

The key contribution of this work is to define a statistical test
of goodness-of-fit, based on a Stein discrepancy computed in a reproducing
kernel Hilbert space (RKHS).  To construct our test statistic, we
use a function class defined by applying the Stein operator to a chosen
space of RKHS functions, as proposed by \citep{OatGirCho15}.\footnote{\citeauthor{OatGirCho15} addressed the problem of variance reduction
in Monte Carlo integration, using the Stein operator to avoid bias. } Our measure of goodness of fit is the largest discrepancy over this
space of functions between empirical sample expectations and target
expectations (the latter being zero, due to the effect of the Stein
operator). The approach is a natural extension to goodness-of-fit
testing of the earlier kernel two-sample tests \citep{gretton2012kernel}
and independence tests \citep{gretton_kernel_2008}, which are based
on the \emph{maximum mean discrepancy}, an integral probability metric.
As with these earlier tests, our statistic is a simple V-statistic,
and can be computed in closed form and in quadratic time. Moreover,
it is an unbiased estimate of the corresponding population discrepancy.
As with all Stein-based discrepancies, only the gradient of the log
target density is needed; we do not require integrals with respect
to the target density -- including the normalisation constant. Given
that our test statistic is a V-statistic, we make use of the extensive
literature on asymptotics of V-statistics to formulate a hypothesis
test \citep{serfling80,leucht_dependent_2013}.\footnote{An alternative linear-time test, based on differences in analytic
functions of the sample and following the recent work of \citep{Chwialkowski2015},
is provided in \citep[Appendix 5.1]{chwialkowski2016kernel}} We provide statistical tests for both uncorrelated and correlated
samples, where the latter is essential if the test is used in assessing
the quality of output of an MCMC procedure.  An identical test was
obtained simultaneously in independent work by \citet{LiuLeeJor16},
for uncorrelated samples.

Several alternative approaches exist in the statistics literature
to goodness-of-fit testing. A first strategy is to partition the space,
and to conduct the test on a histogram estimate of the distribution
\citep{Bar89,Beirlant2,Gyorfi,GyVa02}.\foreignlanguage{english}{
Such space partitioning approaches can have attractive theoretical
properties (e.g. distribution-free test thresholds) and work well
in low dimensions, however they are much less powerful than alternatives
once the dimensionality increases \citep{GreGyo10}.} A second popular
approach has been to use the smoothed $L_{2}$ distance between the
empirical characteristic function of the sample, and the characteristic
function of the target density. This dates back to the test of Gaussianity
of \citet[Eq. 2.1]{BaringhausHenze88}, who used an exponentiated
quadratic  smoothing function. For this choice of smoothing function,
their statistic is identical to the maximum mean discrepancy (MMD)
with the exponentiated quadratic kernel, which can be shown using
the Bochner representation of the kernel \citep[Corollary 4]{SriGreFukLanetal10}.
It is essential in this case that the target distribution be Gaussian,
since the convolution with the kernel (or in the Fourier domain, the
smoothing function) must be available in closed form. An $L_{2}$
distance between Parzen window estimates can also be used  \citep{BowFos93},
giving the same expression again, although the optimal choice of bandwidth
for consistent Parzen window estimates may not be a good choice for
testing \citep{AndHalTit94}. A different smoothing scheme in the
frequency domain results in an energy distance statistic \citep[this likewise being an MMD with a particular choice of kernel; see ][]{SejSriGreFuk13},
which can be used in a test of normality \citep{SzeRiz05}. The key
point is that the required integrals are again computable in closed
form for the Gaussian, although the reasoning may be extended to certain
other families of interest, e.g. \citep{Rizzo09}. The requirement
of computing closed-form integrals with respect to the test distribution
severely restricts this testing strategy. Finally, a problem related
to goodness-of-fit testing is that of model criticism \citep{lloyd2015statistical}.
In this setting, samples generated from a fitted model are compared
via the maximum mean discrepancy with samples used to train the model,
such that a small MMD indicates a good fit. There are two limitation
to the method: first, it requires samples from the model (which might
not be easy if this requires a complex MCMC sampler); second, the
choice of number of samples from the model is not obvious, since too
few samples cause a loss in test power, and too many are computationally
wasteful. Neither issue arises in our test, as we do not require
model samples.

In our experiments, a particular focus is on applying our goodness-of-fit
test to certify the output of approximate Markov Chain Monte Carlo
(MCMC) samplers \citep{Korattikara2014,Welling2011,Bardenet2014}.
These methods use modifications to Markov transition kernels that
improve mixing speed at the cost of introducing asymptotic bias. The
resulting bias-variance trade-off can usually be tuned with parameters
of the sampling algorithms. It is therefore important to test whether
for a particular parameter setting and run-time, the samples are of
the desired quality. This question cannot be answered with classical
MCMC convergence statistics, such as the widely used potential scale
reduction factor (R-factor) \citep{gelman1992inference} or the effective
sample size, since these assume that the Markov chain reaches the
true equilibrium distribution i.e.~absence of asymptotic bias. By
contrast, our test exactly quantifies the asymptotic bias of approximate
MCMC. 

Code can be found at \href{https://github.com/karlnapf/kernel_goodness_of_fit}{https://github.com/karlnapf/kernel\_{}goodness\_{}of\_{}fit}.

\paragraph{Paper outline}

We begin in section \ref{sec:A-Kernel-Goodness-of-fit} with a high-level
construction of the RKHS-based Stein discrepancy and associated statistical
test. In Section \ref{sec:Details}, we provide additional details
and prove the main results. Section \ref{sec:experiment} contains
experimental illustrations on synthetic examples, statistical model
criticism, bias-variance trade-offs in approximate MCMC, and convergence
in non-parametric density estimation.

\selectlanguage{english}%

\section{Test Definition: Statistic and Threshold\label{sec:A-Kernel-Goodness-of-fit}}

We begin with a high-level construction of our divergence measure
 and the associated statistical test. While this section aims to
outline the main ideas, we provide details and proofs in Section \ref{sec:Details}.

\subsection{Stein Operator in RKHS}

Our goal is to write the maximum discrepancy between the target distribution
$p$ and observed sample distribution $q$ in  a \emph{modified}
RKHS, such that functions have zero expectation under $p$. Denote
by ${\cal F}$ the RKHS of real-valued functions on $\mathbb{R}^{d}$
with reproducing kernel $k$, and by ${\cal F}^{d}$ the product RKHS
consisting of elements $f:=(f_{1},\dots,f_{d})$ with $f_{i}\in{\cal F}$,
and with a standard inner product $\left\langle f,g\right\rangle _{\mathcal{F}^{d}}=\sum_{i=1}^{d}\left\langle f_{i},g_{i}\right\rangle _{\mathcal{F}}$.
We further assume that all measures considered in this paper are supported
on an open set, equal to zero on the border, and strictly positive\footnote{An example of such a space is the positive real line}
(so logarithms are well defined). Similarly to \citet{stein1972,gorham2015measuring,OatGirCho15},
we begin by defining a Stein operator $T_{p}$ acting on $f\in\mathcal{F}^{d}$
\[
(T_{p}f)(x):=\sum_{i=1}^{d}\left(\frac{\partial\log p(x)}{\partial x_{i}}f_{i}(x)+\frac{\partial f_{i}(x)}{\partial x_{i}}\right).
\]
Suppose a random variable $Z$ is distributed according to a measure\footnote{\selectlanguage{british}%
Throughout the article, all occurrences of $Z$, e.g. $Z',Z_{i},Z_{\heartsuit}$,
are understood to be distributed according to $q$.\selectlanguage{english}%
} $q$ and $X$ is distributed according to the target measure $p$.
As we will see, the operator can be expressed by defining a function
that depends on gradient of the log-density and the kernel, 
\begin{equation}
\xi_{p}(x,\cdot):=\left[\nabla\log p(x)k(x,\cdot)+\nabla k(x,\cdot)\right],\label{eq:xi}
\end{equation}
whose expected inner product with $f$ gives exactly the expected
value of the Stein operator, 
\[
\ev_{q}T_{p}f(Z)=\langle f,\ev_{q}\xi_{p}(Z)\rangle_{{\cal F}^{d}}=\sum_{i=1}^{d}\langle f_{i},\ev_{q}\xi_{p,i}(Z)\rangle_{{\cal F}},
\]
where $\xi_{p,i}(x,\cdot)$ is the $i$-th component of $\xi_{p}(x,\cdot)$.
For $X$ from the target measure, we have $\ev_{p}(T_{p}f)(X)=0$,
which can be seen using integration by parts, c.f. Lemma \ref{lem:easy}
in the supplement. We can now define a Stein discrepancy and express
it in the RKHS,
\begin{align*}
S_{p}(Z) & :=\sup_{\Vert f\Vert<1}\ev_{q}(T_{p}f)(Z)-\ev_{p}(T_{p}f)(X)\\
 & =\sup_{\Vert f\Vert<1}\ev_{q}(T_{p}f)(Z)\\
 & =\sup_{\Vert f\Vert<1}\langle f,\ev_{q}\xi_{p}(Z)\rangle_{{\cal F}^{d}}\\
 & =\|\ev_{q}\xi_{p}(Z)\|_{{\cal F}^{d}},
\end{align*}
This makes it clear why $\ev_{p}(T_{p}f)(X)=0$ is a desirable property:
we can compute $S_{p}(Z)$ by computing $\|\ev_{q}\xi_{p}(Z)\|$,
without the need to access $X$ in the form of samples from $p$.
To state our first result we define 
\begin{align*}
h_{p}(x,y) & :=\nabla\log p(x)^{\top}\nabla\log p(y)k(x,y)\\
 & \quad+\nabla\log p(y)^{\top}\nabla_{x}k(x,y)\\
 & \quad+\nabla\log p(x){}^{\top}\nabla_{y}k(x,y)\\
 & \quad+\langle\nabla_{x}k(x,\cdot),\nabla_{y}k(\cdot,y)\rangle_{{\cal F}^{d}},
\end{align*}

where the last term can be written as a sum $\sum_{i=1}^{d}\frac{\partial k(x,y)}{\partial x_{i}\partial y_{i}}$.
\foreignlanguage{british}{The following theorem gives a simple closed
form expression for $\|\ev_{q}\xi_{p}(Z)\|_{{\cal F}^{d}}$ in terms
of $h_{p}$.}
\begin{thm}
\label{th:closed_form_discrepancy} If $Eh_{p}(Z,Z)<\infty$, then
$S_{p}^{2}(Z)=\|\ev_{q}\xi_{p}(Z)\|_{{\cal F}^{d}}^{2}=\ev_{q}h_{p}(Z,Z')$,
 where $Z'$ is independent of $Z$ with an identical distribution.\foreignlanguage{british}{ }
\end{thm}
\selectlanguage{british}%

The second main result states that the discrepancy $S_{p}(Z)$ can
be used to distinguish two distributions.

\selectlanguage{english}%
\begin{thm}
\label{theorem_discrepancy_is_metric} Let $q,p$ be probability measures
and $Z\sim q$. If the kernel $k$ is $C_{0}$-universal \citep[Definition 4.1]{carmeli2010vector},
$\ev_{q}h_{q}(Z,Z)<\infty$, and $\ev_{q}\left\Vert \nabla\left(\log\frac{p(Z)}{q(Z)}\right)\right\Vert ^{2}<\infty$,
then $S_{p}(Z)=0$ if and only if $p=q$.
\end{thm}
\selectlanguage{british}%
Section \ref{sec:details_kernel_stein} contains all necessary proofs.
We now proceed to construct an estimator for $S(Z)^{2}$, and outline
its asymptotic properties.

\subsection{Wild Bootstrap Testing}

It is straightforward to estimate the squared Stein discrepancy $S(Z)^{2}$
from samples $\{Z_{i}\}_{i=1}^{n}$: a quadratic time estimator is
a V-Statistic, and takes the form
\[
V_{n}=\frac{1}{n^{2}}\sum_{i,j=1}^{n}h_{p}(Z_{i},Z_{j}).
\]
 The asymptotic null distribution of the normalised V-Statistic $nV_{n}$,
however, has no computable closed form. Furthermore, care has to be
taken when the $Z_{i}$ exhibit correlation structure, as the null
distribution might significantly change, impacting test significance.
The wild bootstrap technique \citep{Shao2010,leucht_dependent_2013,FroLauLerRey12}
addresses both problems. First, it allows us to estimate quantiles
of the null distribution in order to compute test thresholds. Second,
it accounts for correlation structure in the $Z_{i}$ by mimicking
it with an \foreignlanguage{english}{auxiliary} random process: a\foreignlanguage{english}{
simple Markov chain taking values in $\{-1,1\}$, starting from $W_{1,n}=1$,}
\[
W_{t,n}=\mathbf{1}(U_{t}>a_{n})W_{t-1,n}-\mathbf{1}(U_{t}<a_{n})W_{t-1,n},
\]
\foreignlanguage{english}{where the $U_{t}$ are uniform $(0,1)$
i.i.d. random variables and $a_{n}$ is the probability of $W_{t,n}$
changing sign (for i.i.d.~data we set $a_{n}=0.5$). This leads to
a bootstrapped V-statistic }

\selectlanguage{english}%
\emph{
\[
B_{n}=\frac{1}{n^{2}}\sum_{i,j=1}^{n}W_{i,n}W_{j,n}h_{p}(Z_{i,}Z_{j}).
\]
}Proposition \ref{thm:wild_bootstrap_works}\foreignlanguage{british}{
establishes that, under the null hypothesis, $nB_{n}$} is a good
approximation of $nV_{n}$, so it is possible to approximate quantiles
of the null distribution by sampling from it. Under the alternative,
$V_{n}$ dominates $B_{n}$ -- resulting in almost sure rejection
of the null hypothesis.

We propose the following test\foreignlanguage{british}{ procedure
for testing the null hypothesis that the $Z_{i}$ are distributed
according to the target distribution $p$.}
\selectlanguage{british}%
\begin{enumerate}
\item Calculate \foreignlanguage{english}{the test statistic $V_{n}$.}
\selectlanguage{english}%
\item Obtain wild bootstrap samples\foreignlanguage{british}{ }$\{B_{n}\}_{i=1}^{D}$
and estimate the $1-\alpha$ empirical quantile of these samples. 
\item If $V_{n}$ exceeds the quantile, reject.
\end{enumerate}
\selectlanguage{english}%

\section{Proofs of the Main Results\label{sec:Details}}

We now prove the claims made in the previous section.

\subsection{Stein Operator in RKHS}

\label{sec:details_kernel_stein}

Lemma \ref{lem:easy} (in the Appendix) shows that the expected value
of the Stein operator is zero on the target measure.
\begin{proof}[Proof of Theorem \ref{th:closed_form_discrepancy}]

$\xi_{p}(x,\cdot)$ is an element of the reproducing kernel Hilbert
space $\mathcal{F}^{d}$ -- by \citet[Lemma 4.34]{SteChr08} $\nabla k(x,\cdot)\in\mathcal{F}$,
and $\frac{\partial\log p(x)}{\partial x_{i}}$ is just a scalar.
\foreignlanguage{british}{We first show that $h_{p}(x,y)=\langle\xi_{p}(x,\cdot),\xi_{p}(y,\cdot)\rangle.$
Using notations
\begin{align*}
\nabla_{x}k(x,\cdot) & =\left(\frac{\partial k(x,\cdot)}{\partial x_{1}},\cdots,\frac{\partial k(x,\cdot)}{\partial x_{d}}\right)\\
\nabla_{y}k(\cdot,y) & =\left(\frac{\partial k(\cdot,y)}{\partial y_{1}},\cdots,\frac{\partial k(\cdot,y)}{\partial y_{d}}\right),
\end{align*}
we calculate 
\begin{align*}
\langle\xi_{p}(x,\cdot),\xi_{p}(y,\cdot)\rangle & =\nabla\log p(x)^{\top}\nabla\log p(y)k(x,y)\\
 & \quad+\nabla\log p(y)^{\top}\nabla_{x}k(x,y)\\
 & \quad+\nabla\log p(x){}^{\top}\nabla_{y}k(x,y)\\
 & \quad+\langle\nabla_{x}k(x,\cdot),\nabla_{y}k(\cdot,y)\rangle_{\mathcal{F}^{d}}.
\end{align*}
}Next we show that $\xi_{p}(x,\cdot)$ is Bochner integrable \citep[see][Definition A.5.20]{SteChr08},
\[
\ev_{q}\|\xi_{p}(Z)\|_{\mathcal{F}^{d}}\leq\sqrt{\ev_{q}\|\xi_{p}(Z)\|_{\mathcal{F}^{d}}^{2}}=\sqrt{\ev_{q}h_{p}(Z,Z)}<\infty.
\]
\foreignlanguage{british}{This allows us to take the expectation inside
the RKHS inner product. We }next relate the expected value of the
Stein operator to the inner product of $f$ and the expected value
of $\xi_{q}(Z)$, 
\begin{align}
\ev_{q}T_{p}f(Z)=\langle f,\ev_{q}\xi_{p}(Z)\rangle_{\mathcal{F}^{d}} & =\sum_{i=1}^{d}\langle f_{i},\ev_{q}\xi_{p,i}(Z)\rangle_{\mathcal{F}}.\label{eq:expectedVauleT}
\end{align}
We check the claim for all dimensions, 
\begin{align*}
 & \left\langle f_{i},\ev_{q}\xi_{p,i}(Z)\right\rangle _{\mathcal{F}}\\
 & =\left\langle f_{i},\ev_{q}\left[\frac{\partial\log p(Z)}{\partial x_{i}}k(Z,\cdot)+\frac{\partial k(Z,\cdot)}{\partial x_{i}}\right]\right\rangle _{\mathcal{F}}\\
 & =\ev_{q}\left\langle f_{i},\frac{\partial\log p(Z)}{\partial x_{i}}k(Z,\cdot)+\frac{\partial k(Z,\cdot)}{\partial x_{i}}\right\rangle _{\mathcal{F}}\\
 & =\ev_{q}\left[\frac{\partial\log p(Z)}{\partial x_{i}}f_{i}(Z)+\frac{\partial f_{i}(Z,\cdot)}{\partial x_{i}}\right].
\end{align*}
The second equality follows from the fact that a linear operator $\langle f_{i},\cdot\rangle_{\mathcal{F}}$
can be interchanged with the Bochner integral, and the fact that $\xi_{p}$
is Bochner integrable. Using definition of $S(Z)$, Lemma (\ref{lem:easy}),
and Equation (\ref{eq:expectedVauleT}), we have 
\begin{align*}
S_{p}(Z) & :=\sup_{\Vert f\Vert<1}\ev_{q}(T_{p}f)(Z)-\ev_{p}(T_{p}f)(X)\\
 & =\sup_{\Vert f\Vert<1}\ev_{q}(T_{p}f)(Z)\\
 & =\sup_{\Vert f\Vert<1}\langle f,\ev_{q}\xi_{p}(Z)\rangle_{{\cal F}^{d}}\\
 & =\|\ev_{q}\xi_{p}(Z)\|_{\mathcal{F}^{d}}.
\end{align*}
We \foreignlanguage{british}{now calculate closed form expression
for }$S_{p}^{2}(Z)$,

\begin{align*}
S_{p}^{2}(Z) & =\langle\ev_{q}\xi_{p}(Z),\ev_{q}\xi_{p}(Z)\rangle_{\mathcal{F}^{d}}=\ev_{q}\langle\xi_{p}(Z),\ev_{q}\xi_{p}(Z)\rangle_{\mathcal{F}^{d}}\\
 & =\ev_{q}\langle\xi_{p}(Z),\xi_{p}(Z')\rangle_{\mathcal{F}^{d}}=\ev_{q}h_{p}(Z,Z'),
\end{align*}
\foreignlanguage{british}{where $Z'$ is an independent copy of $Z$.}
\end{proof}
\selectlanguage{british}%
Next, we prove that \foreignlanguage{english}{the discrepancy $S$
discriminates different probability measures. }
\selectlanguage{english}%
\begin{proof}[Proof of Theorem \ref{theorem_discrepancy_is_metric}]
 If $p=q$ then $S_{p}(Z)$ is $0$ by Lemma (\ref{lem:easy}). Suppose
$p\neq q$, but $S_{p}(Z)=0$. If $S_{p}(Z)=0$ then, by Theorem \ref{th:closed_form_discrepancy},
$\ev_{q}\xi_{p}(Z)=0.$ In the following we substitute $\log p(Z)=\log q(Z)+[\log p(Z)-\log q(Z)]$,
\begin{align*}
 & \ev_{q}\xi_{p}(Z)\\
 & =\ev_{q}\left(\nabla\log p(Z)k(Z,\cdot)+\nabla k(Z,\cdot)\right)\\
 & =\ev_{q}\xi_{q}(Z)+\ev_{q}\left(\nabla[\log p(Z)-\log q(Z)]k(Z,\cdot)\right)\\
 & =\ev_{q}\left(\nabla[\log p(Z)-\log q(Z)]k(Z,\cdot)\right)
\end{align*}
We have used Theorem \ref{th:closed_form_discrepancy} and Lemma (\ref{lem:easy})
to see that $\ev_{q}\xi_{q}(Z)=0$, since $\|\ev_{q}\xi_{q}(Z)\|^{2}=S_{q}^{2}(Z)=0$.

We \foreignlanguage{british}{recognise} that the expected value of
$\nabla(\log p(Z)-\log q(Z))k(Z,\cdot)$ is the mean embedding of
a function $g(y)=\nabla\left(\log\frac{p(y)}{q(y)}\right)$ with respect
to the measure $q$. By the assumptions the function $g$ is square
integrable; therefore, since the kernel $k$ is $C_{o}$-universal,
by \citet[ Theorem 4.2 b]{carmeli2010vector} its embedding is zero
if and only if $g=0$. This implies that 
\[
\nabla\log\frac{p(y)}{q(y)}=(0,\cdots,0).
\]
A constant vector field of derivatives can only be generated by a
constant function, so $\log\frac{p(y)}{q(y)}=C$, for some $C$, which
implies that $p(y)=e^{C}q(y)$. Since $p$ and $q$ both integrate
to one, $C=0$  and thus $p=q$, which is a contradiction.
\end{proof}
\selectlanguage{british}%

\subsection{Wild Bootstrap Testing}

\selectlanguage{english}%
\label{sub:details_testing}

The two concepts required to derive the distribution of the test statistic
are: $\tau$-mixing \citep{dedecker2007weak,leucht_dependent_2013},
and V-statistics \citep{serfling80}. 

We assume $\tau$-mixing as our notion of dependence within the observations,
since this is weak enough for most practical applications. Trivially,
i.i.d.~observations are $\tau$-mixing. As for Markov chains, whose
convergence we study in the experiments, the property of geometric
ergodicity implies $\tau$-mixing (given that the stationary distribution
has a finite moment of some order -- see the Appendix for further
discussion). For further details on $\tau$-mixing, see \citep{dedecker2005new,dedecker2007weak}.
For this work, we assume a technical condition $\sum_{t=1}^{\infty}t^{2}\sqrt{\tau(t)}\leq\infty$.
 A direct application of \citep[Theorem 2.1]{leucht2012degenerate}
characterises the limiting behavior of $nV_{n}$ for $\tau$-mixing
processes. 
\begin{prop}
\label{thm: null_dist}\textup{If $h$ is Lipschitz continuous and
$\ev_{q}h_{p}(Z,Z)<\infty$ then, under the null hypothesis, $nV_{n}$
converges weakly to some distribution.}
\end{prop}
The proof, which is a simple verification of the relevant assumptions,
can be found in the Appendix. Although a formula for a limit distribution
of \textbf{$V_{n}$} can be derived explicitly \citep[Theorem 2.1][]{leucht2012degenerate},
we do not provide it here. To our knowledge there are no methods of
obtaining quantiles of a limit of \textbf{$V_{n}$} in closed form.
The common solution is to estimate quantiles by a resampling method,
as described in Section \ref{sec:A-Kernel-Goodness-of-fit}. The validity
of this resampling method is guaranteed by the following proposition
(which follows from Theorem 2.1 of \citeauthor{leucht2012degenerate}
 and a modification of the Lemma 5 of \citet{chwialkowski2014wild}),
proved in the supplement.
\begin{prop}
\textup{\label{thm:wild_bootstrap_works}Let $f(Z_{1,n},\cdots,Z_{t,n})=\sup_{x}|P(nB_{n}>x|Z_{1,n},\cdots,Z_{t,n})-P(nV_{n}>x)|$
be a difference between quantiles. If $h$ is Lipschitz continuous
and $\ev_{q}h_{p}(Z,Z)^{2}<\infty$ then, under the null hypothesis,
$f(X_{1,n},\cdots,X_{t,n})$ converges to zero in probability; under
the alternative hypothesis, $B_{n}$ converges to zero, while $V_{n}$
converges to a positive constant.}
\end{prop}
\selectlanguage{british}%
As a consequence, if the null hypothesis is true, we can approximate
any quantile; while under the alternative hypothesis, all quantiles
of $B_{n}$ collapse to zero while $P(V_{n}>0)\to1$. We discuss specific
case of testing MCMC convergence in the Appendix.

\selectlanguage{english}%

\section{Experiments}

\label{sec:experiment}

We provide a number of experimental applications for our test. We
begin with a simple check to establish correct test calibration on
non-i.i.d.~data, followed by a demonstration of statistical model
criticism for Gaussian process (GP) regression. We then apply the
proposed test to quantify bias-variance trade-offs in MCMC, and demonstrate
how to use the test to verify whether MCMC samples are drawn from
the desired stationary distribution. In the final experiment, we move
away from the MCMC setting, and\foreignlanguage{british}{ use the
test to evaluate the convergence of a nonparametric density estimator.
Code can be found at \href{https://github.com/karlnapf/kernel_goodness_of_fit}{https://github.com/karlnapf/kernel\_{}goodness\_{}of\_{}fit}.}

\subsubsection*{Student's t vs.~Normal}

\selectlanguage{british}%
\begin{figure}
\selectlanguage{english}%
\begin{centering}
\includegraphics{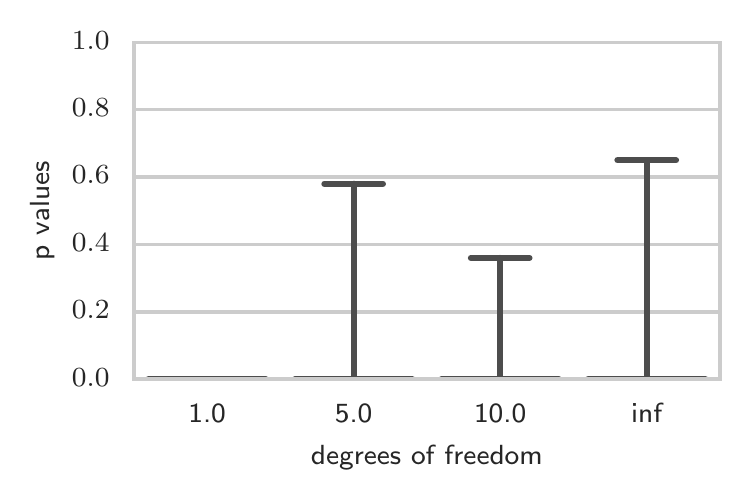}
\par\end{centering}

\selectlanguage{british}%
\caption{\selectlanguage{english}%
Large autocovariance, unsuitable bootstrap. The parameter $a_{n}$
is too large and the bootstrapped V-statistics $B_{n}$ are too low
on average. Therefore, it is very likely that $V_{n}>B_{n}$ and the
test is too conservative. \foreignlanguage{british}{\label{fig:student_bad}}\selectlanguage{british}%
}
\end{figure}

\selectlanguage{english}%
In our first task, we modify Experiment 4.1 from \citet{gorham2015measuring}.
The null hypothesis is that the observed samples come from a standard
normal distribution. We study the power of the test against samples
from a Student's t distribution. We expect to observe low p-values
when testing against a Student's t distribution with few degrees of
freedom. We consider 1, 5, 10 or $\infty$ degrees of freedom, where
$\infty$ is equivalent to sampling from a standard normal distribution.
For a fixed number of degrees of freedom we draw 1400 samples and
calculate the p-value. This procedure is repeated 100 times, and the
bar plots of p-values are shown in Figures \ref{fig:student_bad},\ref{fig:studentst},\ref{fig:thinning}. 

\begin{figure}
\centering{}\includegraphics{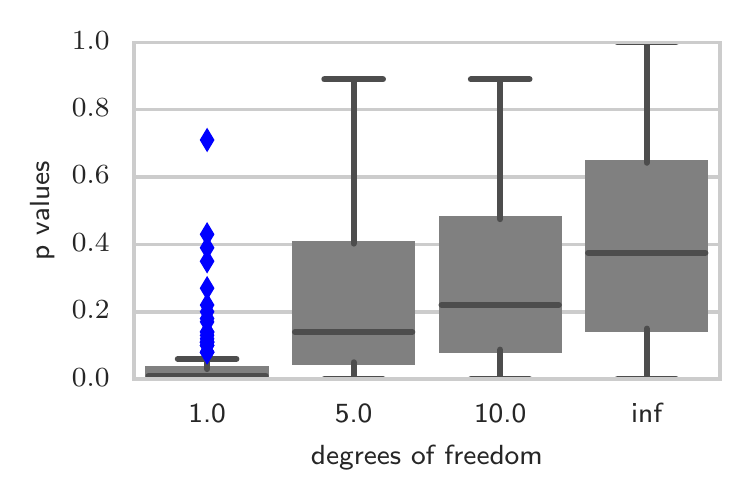} \caption{Large autocovariance, suitable bootstrap. The parameter $a_{n}$is
chosen suitably, but due to a large autocorrelation within the samples,
the power of the test is small (effective sample size is small). \label{fig:studentst}}
\end{figure}

\selectlanguage{british}%
Our twist on \foreignlanguage{english}{the original experiment 4.1
by \citeauthor{gorham2015measuring} is that the draws from the Student's
t distribution exhibit temporal correlation. We generate samples using
a Metropolis\textendash Hastings algorithm, with a Gaussian random
walk with variance 1/2. We emphasise the need for an appropriate choice
of the wild bootstrap process parameter $a_{n}$. In Figure \ref{fig:student_bad}
we plot p-values for $a_{n}$ being set to $0.5$. Such a high value
of $a_{n}$ is suitable for i.i.d.~observations, but results in p-values
that are too conservative for temporally correlated observations.
In Figure \ref{fig:studentst}, we set $a_{n}=0.02$, which gives
a well calibrated distribution of the p-values under the null hypothesis,
however the test power is reduced. Indeed, p-values for five degrees
of freedom are already large. The solution that we recommend is a
mixture of thinning and adjusting $a_{n},$ as presented in the Figure
\ref{fig:thinning}. We thin the observations by a factor of 20 and
set $a_{n}=0.1$, thus preserving both good statistical power and
correct calibration of p-values under the null hypothesis. In a general,
we recommend to thin a chain so that $\text{Cor}(X_{t},X_{t-1})<0.5$,
set $a_{n}=0.1/k$, and run test with at least $\max(500k,d100)$
data points, where $k<10$.}

\begin{figure}
\centering{}\includegraphics{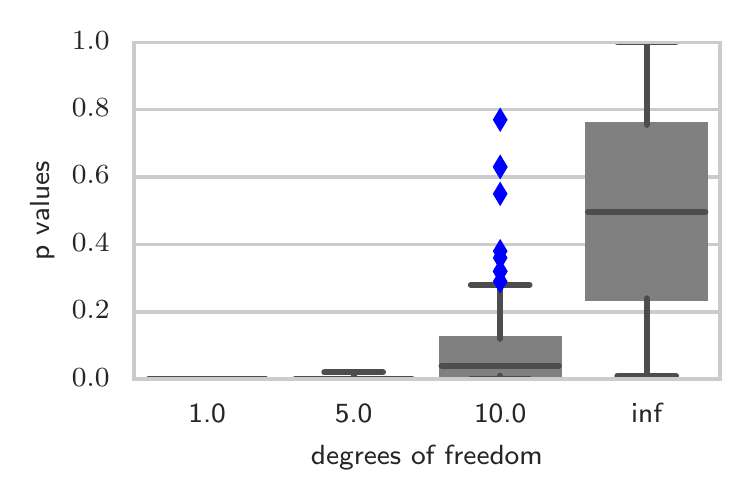}\caption{\selectlanguage{english}%
Thinned sample, suitable bootstrap. Most of the autocorrelation within
the sample is canceled by thinning. To guarantee that the remaining
autocorrelation is handled properly, the wild bootstrap flip probability
is set at $0.1$. \foreignlanguage{british}{\label{fig:thinning}}\selectlanguage{british}%
}
\end{figure}

\subsubsection*{Comparing to a parametric test in increasing dimensions}

In this experiment, we compare with the test proposed by \citet{BaringhausHenze88},
which is essentially an MMD test for normality, i.e.~the null hypothesis
is that $Z$ is a $d$-dimensional standard normal random variable.
We set the sample size to $n=500,1000$ and $a_{n}=0.5$, generate
\[
Z\sim{\cal N}(0,I_{d})\qquad Y\sim U[0,1],
\]
and modify $Z_{0}\leftarrow Z_{0}+Y$. Table \ref{tab:Power-vs-Sample}
shows the power as a function of the sample size. We observe that
for higher dimensions, and where the expectation of the kernel exists
in closed form, an MMD-type test like \citep{BaringhausHenze88} is
a better choice.

\begin{table}
\begin{tabular}{|c|c|c|c|c|c|c|c|}
\hline 
 & \textbf{$d$} & \textbf{2} & \textbf{5} & \textbf{10} & \textbf{15} & \textbf{20} & \textbf{25}\tabularnewline
\hline 
\hline 
\textbf{B\&H} & \multirow{2}{*}{$n=500$} & 1 & 1 & \textbf{1} & \textbf{0.86} & \textbf{0.29} & \textbf{0.24}\tabularnewline
\cline{1-1} \cline{3-8} 
\textbf{Stein} &  & 1 & 1 & 0.86 & 0.39 & 0.05 & 0.05\tabularnewline
\hline 
\hline 
\textbf{B\&H} & \multirow{2}{*}{$n=1000$} & 1 & 1 & 1 & \textbf{1} & \textbf{0.87} & \textbf{0.62}\tabularnewline
\cline{1-1} \cline{3-8} 
\textbf{Stein} &  & 1 & 1 & 1 & 0.77 & 0.25 & 0.05\tabularnewline
\hline 
\end{tabular}

\caption{\label{tab:Power-vs-Sample}Test power vs.~sample size for the test
by \citet{BaringhausHenze88} (B\&H) and our Stein based test.}
\end{table}

\subsubsection*{Statistical Model Criticism on Gaussian Processes}

\begin{figure}
\begin{centering}
\includegraphics{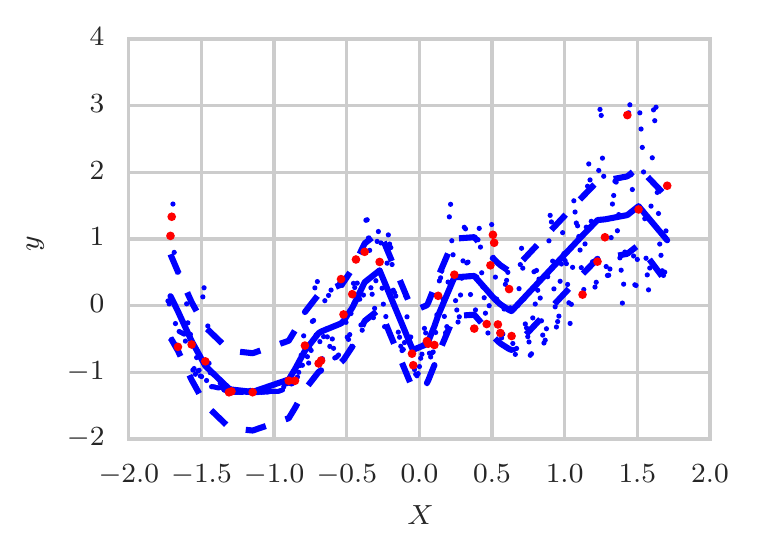}
\par\end{centering}

\caption{Fitted GP and data used to fit (blue) and to apply test (red).\label{fig:experiment_gp_fit}}
\end{figure}

We next apply our test to the problem of statistical model criticism
for GP regression. Our presentation and approach are similar to the
non i.i.d.~case in Section 6 of \citet{lloyd2015statistical}. We
use the \texttt{solar} dataset, consisting of a $d=1$ regression
problem with $N=402$ pairs $(X,y)$. We fit $N_{\text{train}}=361$
data using a GP with an exponentiated quadratic kernel and a Gaussian
noise model, and perform standard maximum likelihood II on the hyperparameters
(length-scale, overall scale, noise-variance). We then apply our test
to the remaining $N_{\text{test}}=41$ data. The test attempts to
falsify the null hypothesis that the \texttt{solar} dataset was generated
from the plug-in predictive distribution (conditioned on training
data and predicted position) of the GP. \citeauthor{lloyd2015statistical}
refer to this setup as non-i.i.d., since the predictive distribution
is a different univariate Gaussian for every predicted point. Our
particular $N_{\text{train}},N_{\text{test}}$ were chosen to make
sure the GP fit has stabilised, i.e.~adding more data did not cause
further model refinement.

Figure \ref{fig:experiment_gp_fit} shows training and testing data,
and the fitted GP. Clearly, the Gaussian noise model is a poor fit
for this particular dataset, e.g.~around $X=-1$. Figure \ref{fig:experiment_gp_test}
shows the distribution over $D=10000$ bootstrapped V-statistics $B_{n}$
with $n=N_{\text{test}}$. The test statistic lies in an upper quantile
of the bootstrapped null distribution, correctly indicating that it
is unlikely the test points were generated by the fitted GP model,
even for the low number of test data observed, $n=41$.

In a second experiment, we compare against \citeauthor{lloyd2015statistical}:
we compute the MMD statistic between test data $(X_{\text{test}},y_{\text{test}})$
and $(X_{\text{test}},y_{\text{rep}})$, where $y_{\text{rep}}$ are
samples from the fitted GP. We draw 10000 samples from the null distribution
by repeatedly sampling new $\tilde{y}_{\text{rep}}$ from the GP plug-in
predictive posterior, and comparing $(X_{\text{test}},\tilde{y}_{\text{rep}})$
to $(X_{\text{test}},y_{\text{rep}})$. When averaged over 100 repetitions
of randomly partitioned $(X,y)$ for training and testing, our goodness
of fit test produces a p-value that is statistically not significantly
different from the MMD method ($p\approx0.1$, note that this result
is subject to $N_{\text{train}},N_{\text{test}}$). We emphasise,
however, that \citeauthor{lloyd2015statistical}'s test requires to
sample from the fitted model (here 10000 null samples were required
in order to achieve stable p-values). Our test \emph{does not} sample
from the GP at all and completely side-steps this more costly approach.

\begin{figure}
\begin{centering}
\includegraphics[scale=0.85]{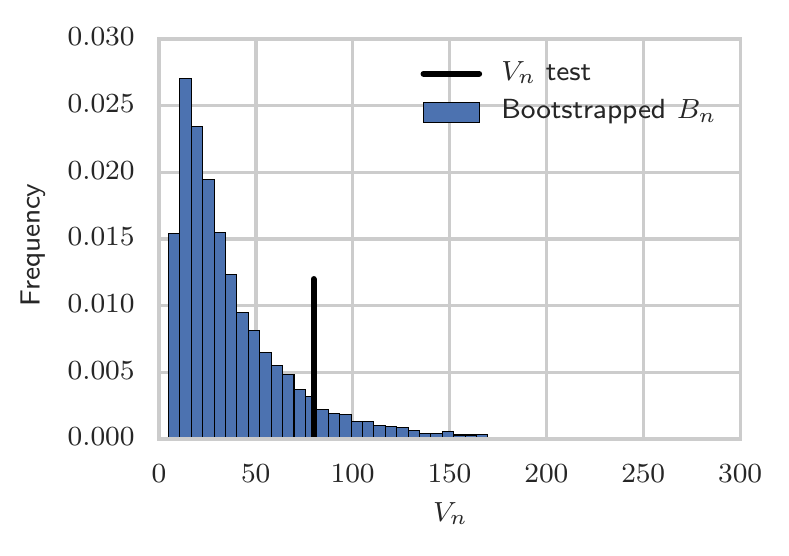}
\par\end{centering}

\caption{Bootstrapped $B_{n}$ distribution with the test statistic $V_{n}$
marked.\label{fig:experiment_gp_test} }
\end{figure}

\selectlanguage{english}%

\subsubsection*{Bias quantification in Approximate MCMC}

We now illustrate how to quantify\foreignlanguage{british}{ bias-variance
trade-offs in an approximate }MCMC algorithm\foreignlanguage{british}{
-- }austerity MCMC \citep{korattikara2013austerity}\foreignlanguage{british}{.
}For the purpose of illustration we use a simple generative model
from \citet{gorham2015measuring,Welling2011}, 
\begin{align*}
\theta_{1}\sim{\cal N}(0,10);\theta_{2}\sim{\cal N}(0,1)\\
X_{i}\sim\frac{1}{2}{\cal N}(\theta_{1},4)+\frac{1}{2}{\cal N}(\theta_{2}+\theta_{1},4) & .
\end{align*}
\begin{figure}
\centering{}\includegraphics{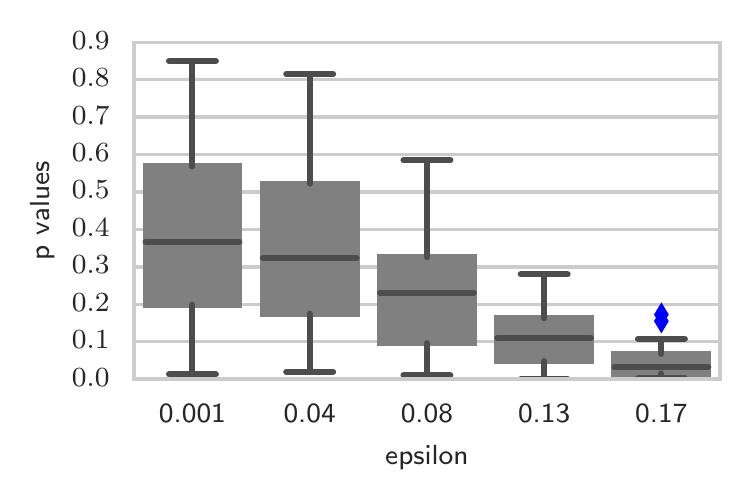}\foreignlanguage{british}{\caption{\selectlanguage{english}%
Distribution of p-values as a function of $\epsilon$ for austerity
MCMC. \label{p-values}\selectlanguage{british}%
}
}
\end{figure}
\foreignlanguage{british}{}
\begin{figure}
\centering{}\includegraphics{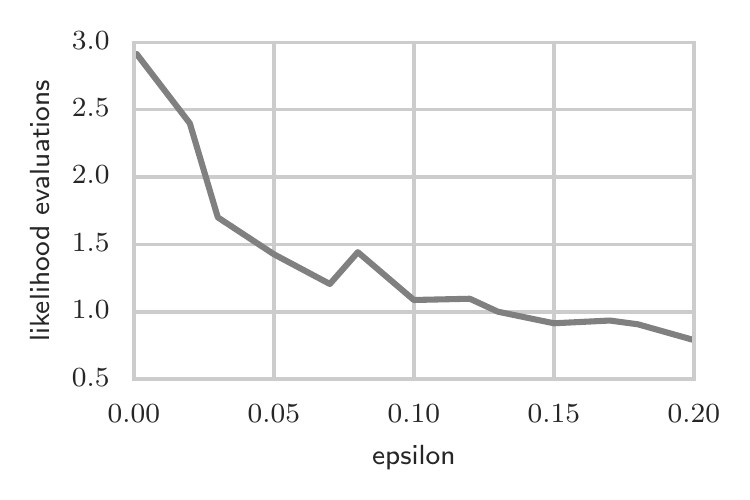}\foreignlanguage{british}{\caption{\selectlanguage{english}%
Average number of likelihood evaluations a function of $\epsilon$
for austerity MCMC (the y-axis is in millions of evaluations). \label{lik-evals}\selectlanguage{british}%
}
}
\end{figure}
Austerity MCMC is a Monte Carlo procedure designed to reduce the number
of likelihood evaluation in the acceptance step of the Metropolis-Hastings
algorithm. The crux of method is to look at only a subset of the data,
and make an acceptance/rejection decision based on this subset. The
probability of making a wrong decision is proportional to a parameter
$\epsilon\in[0,1]$ . This parameter influences the time complexity
of austerity MCMC: when $\epsilon$ is larger, i.e., when there is
a greater tolerance for error, the expected computational cost is
lower. We simulate $\{X_{i}\}_{1\leq i\leq400}$ points from the model
with $\theta_{1}=0$ and $\theta_{2}=1$. In our experiment, there
are two modes in the posterior distribution: one at $(0,1)$ and the
other at $(1,-1)$. We run the algorithm with $\epsilon$ varying
over the range $[0.001,0.2]$. For each $\epsilon$ we calculate an
individual thinning factor, such that correlation between consecutive
 samples from the chains is smaller than $0.5$ (greater $\epsilon$
generally required more thinning). For each $\epsilon$ we test the
hypothesis that $\{\theta_{i}\}_{1\leq i\leq500}$ is drawn from the
true stationary posterior, using our goodness of fit test. We generate
100 p-values for each $\epsilon$ , as shown in Figure \ref{p-values}.
A good approximation of the true stationary distribution is obtained
at $\epsilon=0.4$, which is still parsimonious in terms of likelihood
evaluations, as shown in Figure \ref{lik-evals}. 

\selectlanguage{british}%

\subsubsection*{Convergence in non-parametric density estimation}

\begin{figure}
\begin{centering}
\includegraphics{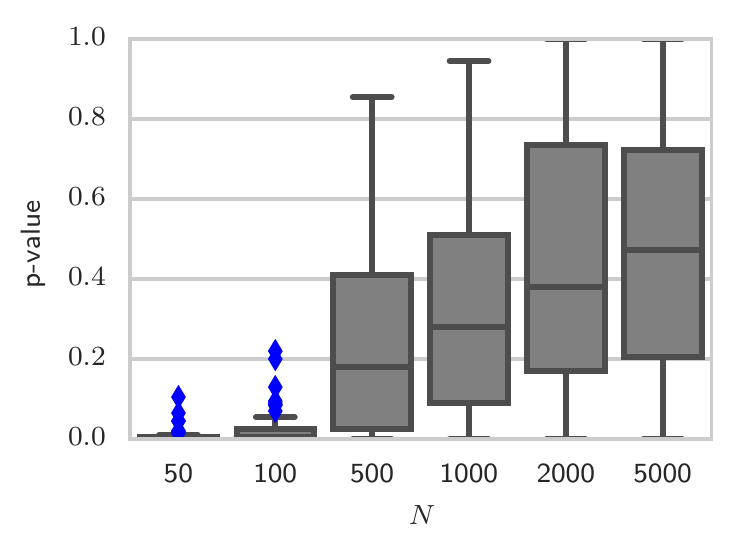}
\par\end{centering}

\caption{Density estimation: p-values for an increasing number of data $N$
for the non-parametric model. Fixed $n=500$.}

\label{fig:density_estimation_increasing_data}
\end{figure}

In our final experiment, we apply our goodness of fit test to measuring
quality-of-fit in nonparametric density estimation. We evaluate two
density models: the infinite dimensional exponential family \citep{SriFukKumGreHyv14},
and a recent approximation to this model using random Fourier features
\citep{strathmann2015gradient}. Our implementation of the model assumes
the log density to take the form $f(x)$, where $f$ lies in an RKHS
induced by a Gaussian kernel with bandwidth $1$. We fit the model
using $N$ observations drawn from a standard Gaussian, and perform
our quadratic time test on a separate evaluation dataset of fixed
size $n=500$. Our goal is to identify $N$ sufficiently large that
the goodness of fit test does not reject the null hypothesis (i.e.,
the model has learned the density sufficiently well, bearing in mind
that it is guaranteed to converge for sufficiently large $N$). Figure
\ref{fig:density_estimation_increasing_data} shows how the distribution
of p-values evolves as a function of $N$; this distribution is uniform
for $N=5000$, but at $N=500$, the null hypothesis would very rarely
be rejected.

We next consider the random Fourier feature approximation to this
model, where the log pdf, $f$, is approximated using a finite dictionary
of random Fourier features \citep{Rahimi2007}. The natural question
when using this approximation is: ``How many random features are
needed?'' Using the same test set size $n=500$ as above, and a large
number of samples, $N=5\cdot10^{4}$, Figure \ref{fig:density_estimation_increasing_features}
shows the distributions of p-values for an increasing number of random
features $m$. From $m=50$, the null hypothesis would rarely be rejected.
Note, however, that the p-values do \emph{not} have a uniform distribution,
even for a large number of random features. This subtle effect is
caused by over-smoothing due to the regularisation approach taken
by \citet[KMC finite]{strathmann2015gradient}, which would not otherwise
have been detected.  

\begin{figure}
\begin{centering}
\includegraphics{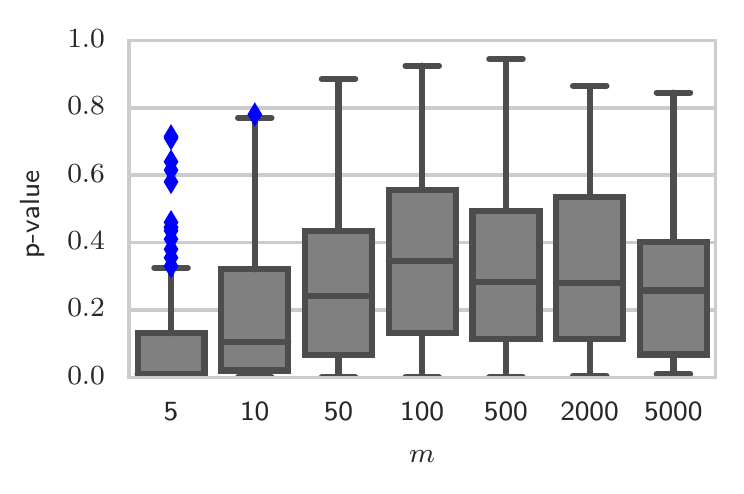}
\par\end{centering}

\caption{Approximate density estimation: p-values for an increasing number
of random features $m$. Fixed $n=500$.}

\label{fig:density_estimation_increasing_features}
\end{figure}

\FloatBarrier

\paragraph{Acknowledgement. }

Arthur Gretton has ORCID 0000-0003-3169-7624.\newpage{}

\bibliographystyle{icml2015}
\bibliography{biblio}

\pagebreak{}

\normalsize\onecolumn

\part*{Appendix}

\selectlanguage{english}%

\section{Proofs}
\begin{lem}
\label{lem:easy} If a random variable $X$ is distributed according
to $p$, under conditions on the kernel

\begin{align*}
0 & =\oint_{\partial\mathcal{X}}k(x,x')q(x)n(x)dS(x'),\\
0 & =\oint_{\partial\mathcal{X}}\nabla_{x}k(x,x')^{\top}n(x')q(x')dS(x'),
\end{align*}
 and then for all $f\in\mathcal{F}$, the expected value of $T$ is
zero, i.e. $\ev_{p}(Tf)(X)=0$.\end{lem}
\begin{proof}
This result was proved on bounded domains $\mathcal{X}\subset\mathbb{R}^{d}$
by \citet[Lemma 1]{OatGirCho15}, where $n(x)$ is the unit vector
normal to the boundary at $x$, and $\oint_{\partial\mathcal{X}}$
is the surface integral over the boundary $\partial\mathcal{X}$.
The case of unbounded domains was discussed by \citet[Remark 2]{OatGirCho15}.
Here we provide an alternative, elementary proof for the latter case.
First we show that the function $p\cdot f_{i}$ vanishes at infinity,
by which we mean that for all dimensions $j$ 
\[
\lim_{x_{j}\to\infty}p(x_{1},\cdots,x_{d})\cdot f_{i}(x_{1},\cdots,x_{d})=0.
\]
The density function $p$ vanishes at infinity. The function $f$
is bounded, which is implied by Cauchy-Schwarz inequality, $\left|f(x)\right|\le\left\Vert f\right\Vert \sqrt{k(x,x)}$.
This implies that the function $p\cdot f_{i}$ vanishes at infinity.
We check that the expected value $\ev_{p}(T_{p})f(X)$ is zero. For
all dimensions $i$, 
\begin{align*}
 & \ev_{p}(T_{p})f(X)\\
 & =\ev_{p}\left(\frac{\partial\log p(X)}{\partial x_{i}}f_{i}(X)+\frac{\partial f_{i}(X)}{\partial x_{i}}\right)\\
 & =\int_{R_{d}}\left[\frac{\partial\log p(x)}{\partial x_{i}}f_{i}(x)+\frac{\partial f_{i}(x)}{\partial x_{i}}\right]p(x)dx\\
 & =\int_{R_{d}}\left[\frac{1}{p(x)}\frac{\partial p(x)}{\partial x_{i}}f(x)+\frac{\partial f(x)}{\partial x_{i}}\right]p(x)dx\\
 & =\int_{R_{d}}\left[\frac{\partial p(x)}{\partial x_{i}}f_{i}(x)+\frac{\partial f_{i}(x)}{\partial x_{i}}p(x)\right]dx\\
 & \overset{(a)}{=}\int_{R_{d-1}}\left(\lim_{R\to\infty}p(x)f_{i}(x)\bigg|_{x_{i}=-R}^{x_{i}=R}\right)dx_{1}\cdots dx_{i-1}\cdots dx_{i+1}\cdots d{x_{d}}\\
 & =\int_{R_{d-1}}0dx_{1}\cdots dx_{i-1}\cdots dx_{i+1}\cdots dx_{d}\\
 & =0.
\end{align*}
For the equation (a) we have used integration by parts, the fact that
$p(x)f_{i}(x)$ vanishes at infinity, and the Fubini-Toneli theorem
to show that we can do iterated integration. The sufficient condition
for the Fubini-Toneli theorem is that $\ev_{q}\langle f,\xi_{p}(Z)\rangle^{2}<\infty$.
This is true since $\ev_{p}\|\xi_{p}(X)\|^{2}\leq\ev_{p}h_{p}(X,X)<\infty$. 
\end{proof}
\selectlanguage{british}%

\selectlanguage{english}%
\begin{proof}[Proof of \foreignlanguage{british}{proposition} \ref{thm: null_dist}]
We check assumptions of Theorem 2.1 from \citep{leucht2012degenerate}.
 Condition A1, $\sum_{t=1}^{\infty}\sqrt{\tau(t)}\leq\infty$, is
implied by assumption $\sum_{t=1}^{\infty}t^{2}\sqrt{\tau(t)}\leq\infty$
in Section \ref{sec:Details}. Condition A2 (iv), Lipschitz continuity
of $h$, is assumed. Conditions A2 i), ii) positive definiteness,
symmetry and degeneracy of $h$ follow from the proof of Theorem (\ref{theorem_discrepancy_is_metric}).
Indeed

\[
h_{p}(x,y)=\langle\xi_{p}(x),\xi_{p}(y)\rangle_{\mathcal{F}^{d}}
\]
\\
so the statistic is an inner product and hence positive definite.
Degeneracy under the null follows from the fact that, by Theorem \ref{th:closed_form_discrepancy},
$\ev_{q}\xi_{p}(Z)=0$. Finally, condition A2 (iii), $\ev_{p}h_{p}(X,X)\leq\infty$,
is assumed.
\end{proof}
\selectlanguage{british}%

\selectlanguage{english}%
\begin{proof}[Proof of \foreignlanguage{british}{proposition \ref{thm:wild_bootstrap_works}}]
We use Theorem 2.1 \citep{leucht2012degenerate} to see that, under
the null hypothesis, $f(Z_{1,n},\cdots,Z_{t,n})$ converges to zero
in probability. Condition A1, $\sum_{t=1}^{\infty}\sqrt{\tau(t)}\leq\infty$,
is implied by assumption $\sum_{t=1}^{\infty}t^{2}\sqrt{\tau(t)}\leq\infty$
in Section \ref{sec:Details}. Condition A2 (iv), Lipschitz continuity
of $h$, is assumed.. Assumption B1 is identical to our assumption
$\sum_{t=1}^{\infty}t^{2}\sqrt{\tau(t)}\leq\infty$ from Section \ref{sec:Details}.
Finally we check assumption B2 (bootstrap assumption):\emph{ $\{W_{t,n}\}_{1\leq t\leq n}$}
is a row-wise strictly stationary triangular array independent of
all $Z_{t}$ such that $\ev W_{t,n}=0$ and $\sup_{n}\ev|W_{t,n}^{2+\sigma}|=1<\infty$
for some $\sigma>0$. The autocovariance of the process is given by
$\ev W_{s,n}W_{t,n}=(1-2p_{n})^{-|s-t|}$, so the function $\rho(x)=\exp(-x)$,
and $l_{n}=\log(1-2p_{n})^{-1}$. We verify that $\lim_{u\to0}\rho(u)=1$.
If we set $p_{n}=w_{n}^{-1}$ , such that $w_{n}=o(n)$ and $\lim_{n\to\infty}w_{n}=\infty$,
then $l_{n}=O(w_{n})$ and $\sum_{r=1}^{n-1}\rho(|r|/l_{n})=\frac{1-(1-2p_{n})^{n+1}}{p_{n}}=O(w_{n})=O(l_{n})$.
 Under the alternative hypothesis, $B_{n}$ converges to zero - we
use \citep[Theorem 2]{chwialkowski2014wild}, where the only assumption
$\ensuremath{\tau(r)=o(r^{-4})}$ is satisfied since $\sum_{t=1}^{\infty}t^{2}\sqrt{\tau(t)}\leq\infty$.
We check the assumption 
\[
\sup_{n}\sup_{i,j<n}\ev_{q}h_{p}(Z_{i},Z_{j})^{2}<\infty.
\]
We have $\ev_{q}h_{p}(Z,Z')^{2}\leq\left(\ev_{q}\|\xi_{p}(Z)\|^{2}\right)^{2}=\left(\ev_{q}h_{p}(Z,Z)\right)^{2}<\infty$
.

We show that under the alternative hypothesis, $V_{n}$ converges
to a positive constant -- using \citep[Theorem 3]{chwialkowski2014wild}.
The zero comportment of $h$ is positive since $S_{p}^{2(Z)}>0$.
We checked the  assumption $\sup_{n}\sup_{i,j<n}\ev_{q}h_{p}(Z_{i},Z_{j})^{2}<\infty$
above.

\end{proof}

\selectlanguage{british}%

\section{MCMC convergence testing\label{sec:MCMC-convergence-testing}}

\paragraph{}

\paragraph{Stationary phase.}

In the stationary phase there are number of results which might be
used to show that the chain is $\tau$-mixing.

\paragraph{Strong mixing coefficients.}

Strong mixing is historically the most studied type of temporal dependence -- a lot of models, including Markov Chains, are proved to be strongly mixing, therefore it's useful to relate weak mixing  to strong mixing. For a random variable $X$ on a probability space $(\Omega,\mathcal{F},P_X)$ and $\mathcal{M} \subset \mathcal{F}$ we define  \begin{equation*} \beta(\mathcal{M},\sigma(X)) = \| \sup_{A \in \mathbb{B}(R)} | P_{X|\mathcal{M}}(A) - P_X(A)|\|. \end{equation*} A process  is called $\beta$-mixing or absolutely regular if   \begin{align*} \beta(r) &= \sup_{l \in \mathbb{N}} \frac 1 l \sup_{ r \leq i_1 \leq ... \leq i_l} \beta( \mathcal F_0,(X_{i_1},...,X_{i_l}) )  \overset{r \to \infty}{\longrightarrow} 0.\ \end{align*}
  \citet{dedecker2005new}[Equation  7.6] relates $\tau$-mixing and $\beta$-mixing , as follows: if $Q_x$ is the generalized inverse of the tail function \[  Q_x(u) = \inf_{t \in R} \{  P(|X| > t) \leq u\},   \] then \[  \tau(\mathcal{M},X) \leq 2 \int_{0}^{\beta(\mathcal{M},\sigma(X))}  Q_x(u) du. \] While this definition can be hard to interpret, it can be simplified in the case $E|X|^p=M$  for some $p>1$, since via Markov's inequality $P(|X|>t) \leq \frac{M}{t^p}$, and thus $\frac{M}{t^p} \leq u $ implies $P(|X|>t) \leq u$. Therefore $Q'(u) = \frac{M}{\sqrt[p]{u}} \geq Q_x(u)$. As a result, we have the  inequality  \begin{equation} \label{eq:theBeta}  \frac{ \sqrt[p]{ \beta(\mathcal{M},\sigma(X))} }{ M }  \geq C  \tau(\mathcal{M},X).  \end{equation} 

\citet{dedecker2005new} provide examples of systems that are $\tau$-mixing.
In particular, given that certain assumptions are satisfied, then
causal functions of stationary sequences, iterated random functions,
Markov chains, and expanding maps are all $\tau$-mixing. 

Of particular interest to this work are Markov chains. The assumptions
provided by \citep{dedecker2005new}, under which Markov chains are
$\tau$-mixing, are difficult to check. We can, however, use classical
theorems about the absolute regularity ($\beta$-mixing). In particular
\citep[Corollary 3.6]{bradley_basic_2005} states that a Harris recurrent
and aperiodic Markov chain satisfies absolute regularity, and \citep[Theorem 3.7]{bradley_basic_2005}
states that geometric ergodicity implies geometric decay of the $\beta$
coefficient. Interestingly \citep[Theorem 3.2]{bradley_basic_2005}
describes situations in which a non-stationary chain $\beta$-mixes
exponentially. 

Using inequalities between $\tau$-mixing coefficient and strong mixing
coefficients, one can use these classical theorems show that e.g for
$p=2$ we have 
\[
\sqrt{\beta(\mathcal{M},\sigma(X))}\geq\tau(\mathcal{M},X).
\]

\end{document}